\documentclass[conference]{IEEEtran}
\IEEEoverridecommandlockouts
\usepackage{cite}
\usepackage{amsmath,amssymb,amsfonts}
\usepackage{amsthm}
\usepackage{graphicx}
\usepackage{textcomp}
\usepackage{xcolor}

\usepackage{hyperref}       
\usepackage{url}            
\usepackage{booktabs}       
\usepackage{amsfonts}       
\usepackage{nicefrac}       
\usepackage{microtype}      

\usepackage{subfigure}
\usepackage{amsmath}
\usepackage{wrapfig}
\usepackage{array}
\usepackage{multirow}
\usepackage{listings}

\usepackage[ruled,linesnumbered]{algorithm2e}

\usepackage{amssymb}
\usepackage[normalem]{ulem}
\useunder{\uline}{\ul}{}

\graphicspath{{./img/}}

\def\BibTeX{{\rm B\kern-.05em{\sc i\kern-.025em b}\kern-.08em
    T\kern-.1667em\lower.7ex\hbox{E}\kern-.125emX}}
\begin{document}

\title{Interdependency Matters: Graph Alignment for Multivariate Time Series Anomaly Detection
}

\author{
\IEEEauthorblockN{
Yuanyi Wang\IEEEauthorrefmark{2},
Haifeng Sun\IEEEauthorrefmark{2}$^{\ast}$ \thanks{* Haifeng Sun and Qi Qi are the corresponding author},
Chengsen Wang\IEEEauthorrefmark{2},
Mengde Zhu\IEEEauthorrefmark{2},
Jingyu Wang\IEEEauthorrefmark{2},\\
Wei Tang\IEEEauthorrefmark{3},
Qi Qi\IEEEauthorrefmark{2}$^{\ast}$,
Zirui Zhuang\IEEEauthorrefmark{2},
Jianxin Liao\IEEEauthorrefmark{2}
}
\IEEEauthorblockA{\IEEEauthorrefmark{2} Beijing University of Posts and Telecommunications, Bejing, China}
\IEEEauthorblockA{\IEEEauthorrefmark{3}Huawei Translation Services Center, Beijing, China}
\IEEEauthorblockA{\{wangyuanyi,hfsun,cswang,arnoldzhu,wangjingyu,qiqi8266,zhuangzirui, liaojx\}@bupt.edu.cn, tangwei133@huawei.com}}

\maketitle

\begin{abstract}
Anomaly detection in multivariate time series (MTS) is crucial for various applications in data mining and industry. Current industrial methods typically approach anomaly detection as an unsupervised learning task, aiming to identify deviations by estimating the normal distribution in noisy, label-free datasets. These methods increasingly incorporate interdependencies between channels through graph structures to enhance accuracy. However, the role of interdependencies is more critical than previously understood, as shifts in interdependencies between MTS channels from normal to anomalous data are significant. This observation suggests that \textit{anomalies could be detected by changes in these interdependency graph series}. To capitalize on this insight, we introduce MADGA (MTS Anomaly Detection via Graph Alignment), which redefines anomaly detection as a graph alignment (GA) problem that explicitly utilizes interdependencies for anomaly detection. MADGA dynamically transforms subsequences into graphs to capture the evolving interdependencies, and Graph alignment is performed between these graphs, optimizing an alignment plan that minimizes cost, effectively minimizing the distance for normal data and maximizing it for anomalous data. Uniquely, our GA approach involves explicit alignment of both nodes and edges, employing Wasserstein distance for nodes and Gromov-Wasserstein distance for edges. To our knowledge, this is the first application of GA to MTS anomaly detection that explicitly leverages interdependency for this purpose. Extensive experiments on diverse real-world datasets validate the effectiveness of MADGA, demonstrating its capability to detect anomalies and differentiate interdependencies, consistently achieving state-of-the-art across various scenarios.
\end{abstract}


\section{Introduction}
\label{section:Introduction}
Anomaly detection in Multivariate Time Series (MTS), which comprises multiple intricate channels, aims at identifying abnormal states and their underlying causes at specific time steps. The challenges in real-world applications \cite{ahmed2017wadi} are compounded by data imbalances and scarce labeled information, often making manual labeling costly and inefficient. Even with some labels, the costly and labor-intensive nature of manual labeling means many anomalies remain undetected. Consequently, unsupervised learning has emerged as the dominant approach in the industry, focusing on detecting anomalies in real noisy data without relying on labels.

Traditional methods like one-class support vector machines \cite{scholkopf2001estimating} and isolation forest \cite{liu2008isolation} have been foundational in this field. Recently, many deep learning methods have emerged, demonstrating notable effectiveness for handling complex data \cite{ruff2019deep}. These methods focus on learning spatial characteristics within multivariate metrics \cite{audibert2020usad} or modeling these temporal dependencies across time steps \cite{wang2024drift}. Despite their success in density estimation and anomaly detection, these approaches often fail to adequately capture the complex interdependencies among the series' channels. The interdependencies between channels are pivotal for detecting anomalies in MTS. For instance, as shown in Fig.\ref{fig_swat}, experiments reveal significant shifts in graph structures, which represent channel interdependencies, from normal to anomalous data. Recent studies have introduced the graph-based model to capture interdependencies within MTS channels, which use fully connected \cite{dai2022graph} or dynamic \cite{zhou2023detecting} graphs across channels. Despite their differences, these methods share a common goal: to leverage interdependencies between channels through graph structures for a more accurate estimation of robust normal distributions.

\begin{figure}[t]
\centering
\includegraphics[width=0.9\linewidth]{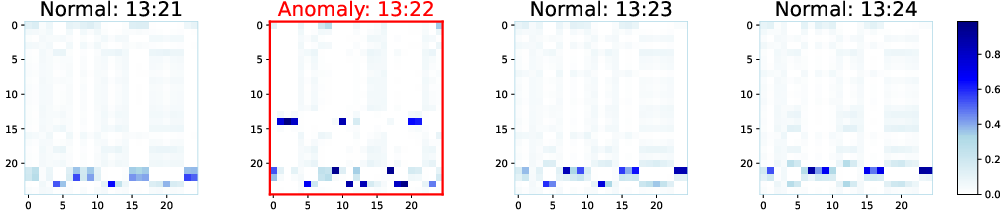}
\caption{The adjacency matrix illustrates the interdependencies of different channels of MTS shift significantly from normal to anomalous data in real-world MTS dataset SWaT.}
\label{fig_swat}
\vspace{-0.4cm}
\end{figure}

In this paper, we introduce MADGA (Multivariate Time Series Anomaly Detection via Graph Alignment), a novel framework that reframes anomaly detection as a Graph Alignment (GA) problem. This approach is motivated by the insight that anomalies manifest through significant shifts in the interdependency graph series from normal to anomalous states. Graph alignment optimizes the mapping between two graphs' nodes to minimize adjacency disagreements. Our framework extends this by aligning not only nodes but also edges, optimizing both distribution and interdependency. The transformation of MTS into a sequence of graphs enables the detection of anomalies by minimizing the alignment cost between graphs, thereby amplifying the distinction between normal and anomalous data. To the best of our knowledge, MADGA is the first to leverage both distribution and interdependency in graph alignment for MTS anomaly detection.

We address several key challenges in realizing this framework:(\romannumeral1) transforming MTS anomaly detection into a graph alignment problem, (\romannumeral2) measuring distribution and interdependency across different spaces, and (\romannumeral3) ensuring real-time performance despite the computational cost of graph alignment. To tackle these, we construct graphs from MTS windows, where channels are nodes and their interdependencies, captured by attention mechanisms, form edges. Anomalous graphs are identified by deviations in topology, and an optimal alignment minimizes costs for normal data while maximizing it for anomalies. We measure deviation using Wasserstein distance (WD) for node alignment and Gromov-Wasserstein distance (GWD) for edge alignment. To enhance efficiency, we employ batch-wise alignments and apply entropic regularization to WD and GWD, utilizing the Sinkhorn algorithm~\cite{peyre2019computational} to further improve computational performance.

We highlight the following contributions of this work:
\begin{itemize}
\item We establish a solid theoretical basis for these concepts that multivariate time series anomaly detection can be transformed into the graph alignment problem.

\item We introduce MADGA, the first framework to leverage both distribution and interdependency of channels in MTS anomaly detection via graph alignment.


\item Our analysis highlights the importance of dynamic graph modeling, demonstrating how shifts in interdependencies are critical for effectively distinguishing anomalous states.
\end{itemize}

\section{Related Work}
\label{section:RelatedWork}

\subsection{MTS Anomaly Detection}
\label{section:AnomalyDetection}
Classical MTS anomaly detection methods, like one-class SVMs \cite{scholkopf2001estimating} and isolation forests \cite{liu2008isolation}, laid the groundwork for the field. With deep learning advancements, new approaches improved complex data handling \cite{ruff2019deep}. Some focus on spatial characteristics, neglecting temporal dependencies \cite{audibert2020usad}, while others, like \cite{wang2024drift, su2019robust}, capture temporal patterns. However, they often miss the intricate interdependencies among data channels.

Recent graph-based methods address this limitation by modeling inter-temporal and inter-variable relationships. They can be categorized into three approaches: reconstruction methods \cite{sabokrou2020deep} that use reconstruction errors to detect anomalies, forecast methods \cite{wang2024drift} that highlight deviations between predicted and actual values, and relational discrepancy methods \cite{dai2022graph, zhou2023detecting} that leverage relationship graphs to estimate robust normal distributions. While these methods effectively model interdependencies, they primarily focus on enhancing encoder performance. Our MADGA introduces a novel perspective by explicitly leveraging graph structural changes to detect anomalies, providing a fundamentally different approach from existing techniques.

\subsection{Graph Alignment (GA)}
\label{section:GraphAlignment}
Graph alignment (GA) optimizes the mapping between node sets of two graphs to minimize adjacency disagreements, often viewed as a special case of the quadratic assignment problem \cite{ganassali2021impossibility}. Traditionally, methods like the Hungarian algorithm \cite{kuhn1955hungarian} addressed computational challenges but struggled with large or isomorphic graphs. With the advent of graph neural networks (GNNs), GA advanced by learning embeddings that capture both topology and node features \cite{li2019graph}. The integration of Wasserstein distance (WD) for node alignment and Gromov-Wasserstein distance (GWD) for metric matching further expanded its utility \cite{peyre2019computational}. In this work, we transform MTS anomaly detection into a GA problem, introducing a novel framework that leverages WD for node alignment and GWD for edge alignment, marking the first application of GA to explicitly utilize both distribution and interdependencies of channels for anomaly detection.

\section{Theoretical Analysis}
\label{section:ProblemStatement}

\subsection{From Detection to Alignment}
\label{section:Definition}

\theoremstyle{plain}
\newtheorem{myDef}{\noindent\bf Definition}
\begin{myDef}
\textbf{Multivariate Time Series (MTS):} The raw time series $D$ primarily consists of unlabeled instances, assumed mostly non-anomalous. Each instance $\mathcal{X} \in D$ contains $B$ sampled series, each with $N$ channels and length $T$, represented as $\mathcal{X} = [X^1, \dots, X^B]^\top$ where $X^i \in \mathbb{R}^{T \times N}$. These are typically prepared via a sliding window of size $T$ and stride $S$, with $B$ indicating the sampling batch size.
\end{myDef}

\begin{myDef}
\textbf{Graph Alignment (GA):} Given two graphs with embeddings $\mathbf{X}_1, \mathbf{X}_2$ and adjacency matrices $\mathbf{A}_1, \mathbf{A}_2$, the GA is generally framed as a quadratic assignment problem:
\begin{equation}
\label{traditionaldefinition}
\begin{aligned}
& \mathop{\arg\max}\limits_{\mathbf{P} \in \Pi} \left \langle \mathbf{P},\mathbf{X} \right \rangle_F + \left \langle \mathbf{P},\mathbf{A} \right \rangle_F
\\  =
& \mathop{\arg\min}\limits \mathcal{D}_1 (\mathbf{X}_1, \mathbf{X_2}) + \mathcal{D}_2 (\mathbf{A}_1, \mathbf{A_2})
\end{aligned}
\end{equation}
where $\mathbf{X} = \mathbf{X}_1 \cup \mathbf{X}_2$, $\mathbf{A} = \mathbf{A}_1 \cup \mathbf{A}_2$, $\mathbf{P}$ are permutation matrices representing alignment plans, $\Pi$ denotes the set of all permutation matrices, $\left \langle \cdot,\cdot\right \rangle$ is Frobenius dot-product, and $\mathcal{D}_i$ measures distance, with lower values indicating higher similarity. Solving this problem is generally NP-hard.
\end{myDef}

\begin{myDef}
     \textbf{Dynamic Graph:} Given a MTS instance $\mathcal{X} \in \mathbb{R}^{B \times T \times N}$, each window series $X^j \in \mathbb{R}^{T \times N}$ forms a dynamic graph where each channel series is regarded as a node with $\mathbf{X}^j_i \in \mathbb{R}^{T \times 1}$ representing node features, and $\mathbf{A}_i \in \mathbb{R}^{N \times N}$ as the learned adjacency matrix. The multivariate time series is transformed into the graph series $(\mathbf{X}_i, \mathbf{A}_i), i=\{ 1,...,B \}$.
\end{myDef} 
Anomalies in MTS often occur in low-density regions, where normal data clusters are more compact, and anomalies are more dispersed. This requires a method to measure distances in MTS, minimizing distances for normal data and maximizing them for anomalies. By transforming MTS into dynamic graphs, we can precisely detect anomalies:
\begin{equation}
    I = \mathop{\arg\max}\limits_{i \in \{1,2,...,B\}} \mathcal{D}_1 (\mathbf{X}^i, \mathbf{X}^\Omega) + \mathcal{D}_2 (\mathbf{A}^i, \mathbf{A}^\Omega)
\end{equation}
Given any sub-series $X^i \in \mathcal{X}$, we exclude this series to define the remaining part of the instance as
\begin{equation}
    X^\Omega:= \mathcal{X}\backslash X^i,
\end{equation}
and the adjacency matrices and graph embeddings are defined as follows, where $d$ is the embedding dimension:
\begin{equation}
\begin{aligned}
    & \mathbf{X}^\Omega:= \mathbf{X}\backslash \mathbf{X}^i \in \mathbb{R}^{(B-1)\times N \times d} \\ & \mathbf{A}^\Omega:= \mathbf{A}\backslash \mathbf{A}^i \in \mathbb{R}^{(B-1)\times N \times N}
\end{aligned}
\end{equation}

\begin{myDef}
\label{defAD}
\textbf{Anomaly Detection:}
Given MTS data $D$, we construct dynamic graphs for each instance $\mathcal{X}$, measuring distributional and topological differences using Wasserstein distance $\mathcal{D}_{wd}$ for node embeddings and Gromov-Wasserstein distance $\mathcal{D}_{gwd}$ for adjacency matrices. The objective is:
\begin{equation}
\label{ourdefinition}
\begin{aligned}
\mathop{\arg\min}\limits & \quad \mathcal{D}_{wd} (\mathbf{X}^i, \mathbf{X}^\Omega) + \mathcal{D}_{gwd} (\mathbf{A}^i, \mathbf{A}^\Omega) \\
& = \mathop{\arg\max}\limits_{\mathbf{P} \in \Pi} \left \langle \mathbf{P},\mathbf{X} \right \rangle_F + \left \langle \mathbf{P},\mathbf{A} \right \rangle_F \\
& = \mathop{\arg\min}\limits_{\mathbf{P} \in \Pi} \left| \right| \mathbf{P} \mathbf{A}^i \mathbf{X}^i - \mathbf{A}^\Omega \mathbf{X}^\Omega \left| \right|^2_F
\end{aligned}
\end{equation}
\end{myDef}
In the unsupervised setting, anomalies are detected by evaluating GA distances, where lower distances indicate anomalies. Unlike typical GA, our alignment is between one graph and the remaining graphs in the same instance.

\subsection{Conditional Alignment}
Definition \ref{defAD} reframes MTS anomaly detection as a graph alignment problem, aiming to minimize distances for normal data and achieve optimal alignment. Importantly, Equation \ref{ourdefinition} treats node embeddings $\mathbf{X}$ and adjacency matrices $\mathbf{A}$ as dependent, where $\mathbf{X}$ is often conditional on $\mathbf{A}$. This section explores node alignment in greater detail to better implement the alignment process described in Equation \ref{ourdefinition}.
\newtheorem{theorem}{\bf Theorem}
\begin{theorem}
\label{theorem1}
Given a dynamic graph generated from an MTS instance $\mathcal{X} \in \mathbb{R}^{B \times N \times T}$ with adjacency matrices $\mathbf{A} \in \mathbb{R}^{B \times N \times N}$ and node embeddings $\mathbf{X} \in \mathbb{R}^{B \times N \times d}$ (where $d$ is the embedding dimension), the graph alignment between source and target graphs can be transformed into:
\begin{equation}
\mathop{\arg\max}\limits_{\mathbf{P} \in \Pi} \left \langle \mathbf{P}, \mathbf{A}^i \mathbf{X}^i  (\mathbf{A}^\Omega \mathbf{X}^\Omega)^\top \right \rangle_F
\end{equation}
\end{theorem}
\begin{proof}
Using the Frobenius norm property $\|A - B\|_F^2 = \|A\|_F^2 + \|B\|_F^2 - 2\langle A, B \rangle_F$, we can reformulate Equation \ref{ourdefinition} as follows:
\begin{equation}
\begin{aligned}
& \mathop{\arg\min}\limits_{\mathbf{P} \in \Pi} \| \mathbf{P} \mathbf{A}^i \mathbf{X}^i - \mathbf{A}^\Omega \mathbf{X}^\Omega \|_F^2 = \\
\mathop{\arg\min}\limits_{\mathbf{P} \in \Pi} & \left( \|\mathbf{P} \mathbf{A}^i \mathbf{X}^i\|_F^2 + \|\mathbf{A}^\Omega \mathbf{X}^\Omega\|_F^2 - 2 \langle \mathbf{P} \mathbf{A}^i \mathbf{X}^i, \mathbf{A}^\Omega \mathbf{X}^\Omega \rangle_F \right)
\end{aligned}
\end{equation}
Given $\mathbf{P}$ is orthogonal \cite{ganassali2021impossibility}, $|\mathbf{P} \mathbf{A}^i \mathbf{X}^i|F^2$ and $|\mathbf{A}^\Omega \mathbf{X}^\Omega|F^2$ remain constants. This reduces the problem to maximizing:
\begin{equation}
\mathop{\arg\min}\limits_{\mathbf{P} \in \Pi} \langle \mathbf{P} \mathbf{A}^i \mathbf{X}^i, \mathbf{A}^\Omega \mathbf{X}^\Omega \rangle_F
\end{equation}
For arbitrarily real matrices $A$ and $B$, these two equations always hold: $\langle A, B \rangle_F = \operatorname{Tr}(AB^T)$ and $\langle A, B + C \rangle_F = \langle A, B \rangle_F + \langle A, C \rangle_F$, where $\operatorname{Tr}(X)$ represents the trace of matrix $X$. Therefore, Theorem.\ref{theorem1} is proved:
\begin{equation}
\begin{aligned}
& \mathop{\arg\min}\limits_{\mathbf{P} \in \Pi} \langle \mathbf{P} \mathbf{A}^i \mathbf{X}^i, \mathbf{A}^\Omega \mathbf{X}^\Omega \rangle_F \\
= & \mathop{\arg\min}\limits_{\mathbf{P} \in \Pi} \operatorname{Tr}(\mathbf{P} \mathbf{A}^i \mathbf{X}^i, \mathbf{A}^\Omega \mathbf{X}^\Omega) \\
= & \mathop{\arg\min}\limits_{\mathbf{P} \in \Pi} \langle \mathbf{P}, \mathbf{A}^i \mathbf{X}^i (\mathbf{A}^\Omega \mathbf{X}^\Omega)^\top \rangle_F
\end{aligned}
\end{equation}
\end{proof}
Theorem \ref{theorem1} establishes that node alignment must consider the distribution of node embeddings $\mathbf{X}$ conditioned on adjacency matrices $\mathbf{A}$, rather than just initial node features. This approach enables the use of any graph propagation operator, such as graph convolution, for conditional alignment. Accordingly, we propose a conditional distribution encoder (Section \ref{subsectionencoder}) that encodes initial node features and adjacency matrices into node embeddings, forming the basis for node alignment.

\subsection{Problem Statement}
\newtheorem{myProblem}{\noindent\bf Problem}
\begin{myProblem}
Given an unlabeled dataset $D = {\mathcal{X}_i}^{|D|}{i=1}$ of MTS, our objective is to construct dynamic graph by learning the interdependencies between channels (adjacency matrix $\mathbf{A}$) and encoding the MTS $\mathcal{X}$ into node embeddings $\mathbf{X}$. We then compute graph alignment distances, $\mathcal{D}_{wd}$ and $\mathcal{D}_{gwd}$, to derive anomaly scores for each instance $\mathcal{X}$.
\begin{equation}
\begin{aligned}
    Encoder(\mathcal{X}) &\rightarrow \mathbf{A},\mathbf{X} \\
    Alignment(\mathbf{X}, \mathbf{A}) &\rightarrow \mathcal{D}_{wd}, \mathcal{D}_{gwd}
\end{aligned}
\end{equation}
\end{myProblem}

\section{Method}
\label{section:Method}
This section introduces the MADGA framework for MTS anomaly detection, which reformulates the problem as graph alignment and uses GA distances as anomaly scores. MADGA includes a dynamic graph construction module to model interdependencies as adjacency matrices $\mathbf{A}$ and a conditional distribution encoder for generating graph embeddings $\mathbf{X}$. Wasserstein distance (WD) aligns nodes, while Gromov-Wasserstein distance (GWD) aligns edges. Fig. \ref{MADGA_structure} illustrates the MADGA framework.

\begin{figure}[t]
\centering
\includegraphics[width=\linewidth]{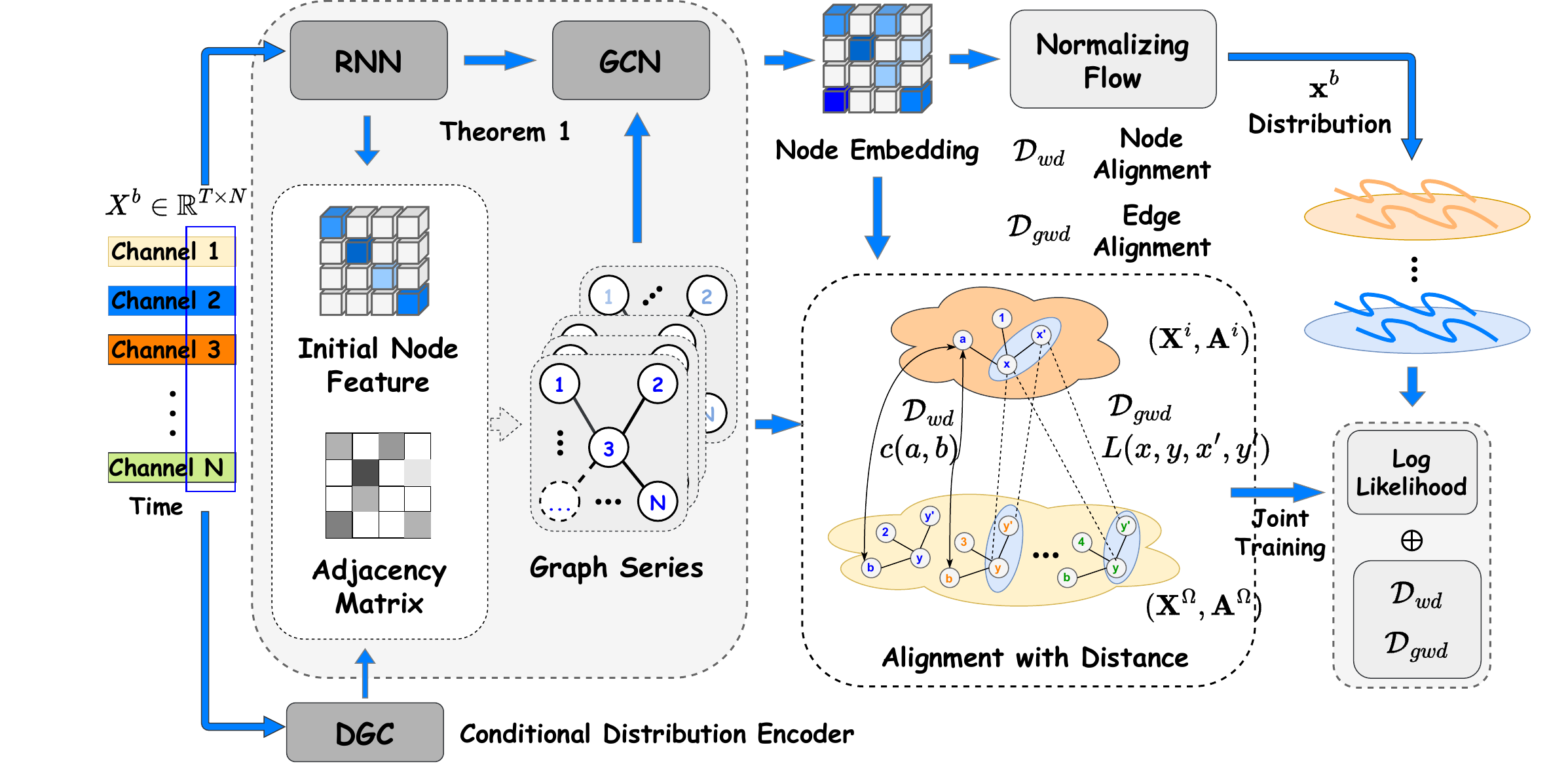}
\caption{Overview of the proposed MADGA.}
\label{MADGA_structure}
\vspace{-0.4cm}
\end{figure}

\subsection{Dynamic Graph Construction (DGC)}
\label{section:graphconstruction}
To account for the mutual and evolving interdependencies among channels, self-attention is employed to derive a dynamic graph structure. Each channel in MTS is treated as a graph node. For each window sequence $X^b = [x_1^b, x_2^b, \ldots, x_n^b] \in \mathcal{X}, x_i^b \in \mathbb{R}^{T \times 1}$, the query and key for node $i$ are given by $W^Q x_i^b$ and $W^K x_i^b$ respectively, where $W^Q, W^K \in \mathbb{R}^{T \times T}$ are learned matrices. The attention score $a_{ij}^b$ quantifying the relation from node $i$ to node $j$ is computed as:
\begin{equation}
    a_{ij}^b = \frac{e_{ij}}{\sum_{j=1}^N e_{ij}}
\end{equation}
\begin{equation}
    e_{ij}^b = exp(\frac{(W^Q x_i^b)(W^K x_i^b)^\top}{\sqrt{T}})
\end{equation}
This attention matrix forms the adjacency matrix $\mathbf{A}^b$, dynamically adapting over time to capture evolving interdependencies.

\subsection{Conditional Distribution Encoder}
\label{subsectionencoder}
Estimating the distribution of MTS, conditioned on the adjacency matrix $\mathbf{A}$, requires robust spatio-temporal information. The dynamic graph captures spatial and temporal correlations essential for feature time series. Following \cite{zhou2023detecting, dai2022graph}, we use RNNs, like LSTM \cite{hochreiter1997long}, to model temporal dynamics:
\begin{equation}
\label{equation:rnn}
    H_n^t = \mathbf{RNN}(X_n^t, H_n^{t-1}),\quad t\in [BS:BS+T]
\end{equation}
To enhance temporal relationships, we apply graph convolution over the dynamically learned graph $\mathbf{A}$, combining current and historical node information:
\begin{equation}
    \mathbf{X}^t = ReLU(\mathbf{A}^b H^t W_1 + H^{t-1} W_2)W_3
\end{equation}
Here, $W_1$, $W_2$, and $W_3$ are learned weights. The final conditional graph embedding for instance $\mathcal{X}$ is obtained by concatenating $\mathbf{X}^t$ along the time dimension, summarizing the spatio-temporal context for anomaly detection.

\subsection{Alignment Distance}
As shown in Fig. \ref{MADGA_structure}, the goal of graph alignment is to find an optimal plan that minimizes alignment cost, reducing the distance for normal data and increasing it for anomalies. We employ two distances: Wasserstein distance (WD) for node and Gromov-Wasserstein distance (GWD) for edge alignment.

Let $P(X)$ and $P(Y)$ represent probability measures on metric spaces $X$ and $Y$. For two discrete distributions,$\mu = \sum_{i=1}^n u_i \delta_{x_i}$ and $\nu = \sum_{j=1}^m v_j \delta_{y_j}$, the definition as:
\begin{myDef}
\label{Definition:WD}
\textbf{Wasserstein distance (WD):}
The $p$-Wasserstein distance between $\mu$ and $\nu$ is formally defined as:
\begin{equation}
\begin{aligned}
    D_{wd}(\mu, \nu) = \left( \min_{\mathbf{P} \in \Pi(\mu, \nu)} \sum_{i=1}^n \sum_{j=1}^m \mathbf{P}_{ij} \cdot c(x_i, y_j)^p \right)^{1/p}
\end{aligned}
\end{equation}
\end{myDef}
where $\mathbf{P}$ is the permutation matrix minimizing the assignment cost $c(x, y)$. In our GA framework, WD measures distances between sampled node pairs, typically with $p=1$.

In contrast to directly calculating node distances as in WD, Gromov-Wasserstein distance (GWD) allows the alignment of edges by comparing node pair distances in the adjacency matrices. It is defined as:
\begin{myDef}
\label{Definition:GWD}
\textbf{Gromov-Wasserstein distance (GWD):}
Consider two metric measure spaces, where \(\mu \in P(X)\) and \(\nu \in P(Y)\) represent probability measures over discrete spaces \(X\) and \(Y\). The $p$-Gromov-Wasserstein distance, a measure of dissimilarity between these spaces, is formally defined as follows:
\begin{equation}
\begin{aligned}
D_{gwd}(\mu, \nu) = \left( \min_{\mathbf{P} \in \Pi(\mu, \nu)} \sum_{i,j,i',j'} \hat{\mathbf{P}}_{ij} \hat{\mathbf{P}}_{i'j'} L(x_i, y_j, x_{i'}, y_{j'})^p \right)^{1/p}
\end{aligned}
\end{equation}
\end{myDef}
\((x_i, x_{i'})\) and \((y_j, y_{j'})\) are interpreted as node pairs in the dual graphs. The optimized matrix $\hat{\mathbf{P}}$  aligns edges across the graphs, typically with $p=1$.

\subsection{Graph Alignment via Distance}
To overcome the independent limitations of WD and GWD—WD aligns nodes but overlooks edge dependencies, while GWD focuses on edges without directly addressing node representations—we propose an integrated approach that optimizes both:
\begin{equation}
\begin{aligned}
    \mathop{\arg\min}\limits_{\mathbf{P}, \hat{\mathbf{P}}\in \Pi} \mathcal{D} (\mathbf{X}^i, \mathbf{P}\mathbf{X}^\Omega) + \mathcal{D} (\mathbf{A^i}, \hat{\mathbf{P}}A^\Omega\hat{\mathbf{P}}^\top) \\
    = D_{wd}(\mu, \nu) + D_{gwd}(\mu, \nu) = \mathcal{D}_{GA}(\mu, \nu)
\end{aligned}
\end{equation}
Formally, the proposed graph alignment distance is defined as:
\begin{equation}
\begin{aligned}
    \mathcal{D}_{GA}(\mu, \nu) = & \min_{\mathbf{P},\hat{\mathbf{P}} \in \Pi(u,v)} \lambda  \sum_{i,j,i',j'} \left( \mathbf{P}_{ij}  c(x_i, y_j) \right. \\
    & \left. + \quad \hat{\mathbf{P}}_{ij} \hat{\mathbf{P}}_{i'j'} L(x_i, y_j, x_{i'}, y_{j'}) \right)
\end{aligned}
\end{equation}
Here, $\lambda$ balances node and edge alignments. Given the NP-hard nature of WD and GWD, we apply entropic regularization to facilitate computation. Using the Envelope Theorem \cite{peyre2019computational}, we employ the Sinkhorn algorithm to solve for WD with regularization:
\begin{equation}
\min_{\mathbf{P} \in \Pi(u,v)} \left(\sum_{i=1}^n \sum_{j=1}^m \mathbf{P}_{ij} c(x_i, y_j) + \beta H(\mathbf{P})\right),
\end{equation}
where \( H(\mathbf{P}) = \sum_{i,j} \mathbf{P}_{ij} \log \mathbf{P}_{ij} \), and \( \beta \) is the hyper-parameter controlling the importance of the entropy term. 

By incorporating the Sinkhorn operator, this method accelerates computational efficiency, making $\mathcal{D}_{GA}$ practical for deep learning frameworks like PyTorch and TensorFlow.

\subsection{Graph Normalizing Flow}
We implement a normalizing flow $f^{\theta}(x \mid \mathcal{C})$ for improved density estimation, following GANF \cite{dai2022graph} and MTGFlow \cite{zhou2023detecting}. Here, \( x \) is the input sequence, and \( \mathcal{C} \) is the condition distribution from the GCN output. Assuming a multivariate Gaussian distribution with identity covariance, the density estimation for channel \( n \) is:
\begin{equation}
\begin{aligned}
    \mathcal{L}_{f} & = \log \left( P_{x_n}(x_n) \right) =\log \left(  P_{z_n}(f^{\theta}(x_n \mid \mathcal{C})) \left| \det \left( \frac{\partial f^{\theta}}{\partial x_n^T} \right) \right|\right) \\
    & \approx \frac{1}{BN} \sum_{b=1}^B \sum_{n=1}^N \left( -\frac{1}{2} \left\| \mathbf{x}_n - \mu_n \right\|^2 + \log \left| \det \left(\frac{\partial f^{\theta}_n}{\partial x_n^b T}\right) \right| \right)
\end{aligned}
\end{equation}

\subsection{Joint Training and Measure}
\subsubsection{Joint Training}
MADGA integrates graph structure learning with RNN-based spatio-temporal modeling. WD and GWD align dynamic graphs, differentiating normal from anomalous data. The total loss is:
\begin{equation}
    \mathcal{L} = \mathcal{D}_{GA} - \mathcal{L}_{NF}
\end{equation}
Optimization minimizes GA distance while maximizing density log-likelihood, focusing on low-density anomalies.
\begin{equation}
\begin{aligned}
    \arg\min \sum_{i,j,i',j'} \lambda &  \left( \mathbf{P}_{ij}  c(x_i, y_j) \right. + \\ \left. \quad \hat{\mathbf{P}}_{ij} \hat{\mathbf{P}}_{i'j'} L(x_i, y_j, x_{i'}, y_{j'}) \right)  & + \arg\max\log \left( P_{\mathcal{X}}(x) \right)
\end{aligned}
\end{equation}

\subsubsection{Anomaly Measure}
Anomalies are detected based on low log-likelihoods. The anomaly score \( S_b \) for sequence \( X^b \) combines graph alignment distance with average negative log-likelihood:
\begin{equation}
S_b = \mathcal{D}_{GA} - \frac{1}{N} \sum_{n=1}^N \log(P_{x_n}(x_n^b))
\end{equation}
Higher scores suggest greater abnormality. Thresholds are set by the interquartile range (IQR) of training scores, with \( Q_1 \) and \( Q_3 \) are the 25-th and 75-th percentile of \( S_b \).
\begin{equation}
\text{Threshold} = Q_3 + 1.5 \times (Q_3 - Q_1),
\end{equation}

\section{Experiment}
\label{section:Experiment}
\subsection{Experiment Setting}
\subsubsection{Dataset}
We evaluate MADGA on five real-world MTS datasets with labeled anomalies: SWaT \cite{goh2017dataset}, WADI \cite{ahmed2017wadi}, PSM \cite{abdulaal2021practical}, MSL \cite{hundman2018detecting}, and SMD \cite{su2019robust}. These datasets span industrial sensors and server metrics. Following prior works \cite{dai2022graph, zhou2023detecting}, we split each dataset into 60\% for training and 40\% for testing. Dataset details are summarized in Table \ref{tab:dataset}.

\begin{table}[t]
\centering
\caption{The static and settings of four public datasets.}
\label{tab:dataset}
\large
\resizebox{\linewidth}{!}{
\begin{tabular}{@{}ccccccc@{}}
\toprule
Dataset & \begin{tabular}[c]{@{}c@{}}Channel\\ number\end{tabular} & \begin{tabular}[c]{@{}c@{}}Training\\ set\end{tabular} & \begin{tabular}[c]{@{}c@{}}Training set\\ anomaly ratio (\%)\end{tabular} & \begin{tabular}[c]{@{}c@{}}Testing\\ set\end{tabular} & \begin{tabular}[c]{@{}c@{}}Testing set\\ anomaly ratio (\%)\end{tabular} \\ 
\midrule
SWaT    & 51                     & 269951       & 17.7                             & 89984       & 5.2                             \\
WADI    & 123                     & 103680       & 6.4                             & 69121       & 4.6                             \\
PSM     & 25                     & 52704        & 23.1                             & 35137       & 34.6                            \\
MSL     & 55                     & 44237        & 14.7                             & 29492       & 4.3                             \\
SMD     & 38                     & 425052       & 4.2                              & 283368      & 4.1                             \\ 
\bottomrule
\end{tabular}}
\vspace{-0.4cm}
\end{table}

\subsubsection{Baselines and Details}
We compare MADGA with state-of-the-art semi-supervised methods (DeepSAD \cite{ruff2019deep}, DeepSVDD \cite{ruff2018deep}, ALOCC \cite{sabokrou2020deep}, DROCC \cite{goyal2020drocc}, USAD \cite{audibert2020usad}) and unsupervised methods (DAGMM \cite{zong2018deep}, GANF \cite{dai2022graph}, MTGFlow \cite{zhou2023detecting}). Experiments are conducted with window sizes (60 for PSM and SMD; 80 and 100 for SWaT and MSL) and a stride of 10. Optimization uses Adam at a 0.002 learning rate. The architecture includes a single LSTM layer for temporal features and a self-attention layer (dropout ratio of 0.2) for dynamic graph learning. The normalizing flow model uses a MAF configuration, with batch sizes of 256 (512 for SWaT). The alignment distance weight, $\lambda$, is set at 0.1. All experiments run for 60 epochs on NVIDIA A100 and A800 GPUs. The code and data is available in \url{https://github.com/wyy-code/MADGA}

\subsection{Main Results}

\begin{table}[htbp]
\caption{Main results: AUC-ROC (\%) of anomaly detection.}
\label{tab1}
\renewcommand\arraystretch{1.3}
\setlength{\tabcolsep}{4pt}
\centering
\resizebox{!}{!}{
\begin{tabular}{c|ccccc}
\toprule
Dataset  & SWaT              & WADI              & PSM               & MSL               & SMD               \\
\midrule
DeepSVDD & 72.8±3.4          & 85.9±1.8          & 65.3±6.6          & 61.2±5.4          & 75.5±15.5         \\
ALOCC    & 58.3±2.1          & 80.6±4.7          & 63.8±2.7          & 53.1±1.1          & 80.5±11.1         \\
DROCC    & 75.2±4.6          & 75.9±2.1          & 72.3±2.2          & 53.4±1.6          & 76.7±8.7          \\
DeepSAD  & 75.4±2.4          & 79.4±5.2          & 73.2±3.3          & 61.6±0.6          & 85.9±11.1         \\
USAD     & 78.8±1.0          & 86.1±0.9          & 78.0±0.2          & 57.0±0.1          & 86.9±11.7         \\
DAGMM    & 72.8±3.0          & 77.2±0.9          & 64.6±2.6          & 56.5±2.6          & 78.0±9.2          \\
GANF     & 82.4±1.2          & 90.7±0.7          & 82.3±1.6          & 64.2±1.5          & 90.3±7.4          \\
MTGFlow  & 84.2±1.4          & 91.9±1.1          & 84.9±2.4          & 66.4±1.8          & 91.5±7.2          \\
\midrule
MADGA    & \textbf{88.9±0.7} & \textbf{92.0±1.0} & \textbf{87.2±1.4} & \textbf{68.1±1.6} & \textbf{92.1±6.7} \\
\bottomrule
\end{tabular}}
\vspace{-0.3cm}
\end{table}


MADGA was evaluated against seven baselines using the AUC-ROC metric, with results summarized in Table \ref{tab1}. For the SMD dataset, we averaged performance across its 28 subsets. MADGA consistently outperformed all baselines in unsupervised settings, yielding the following insights: (\romannumeral1) Methods like DeepSVDD and DROCC, which project samples into a hypersphere, often fail to define accurate decision boundaries for anomaly detection, while generative-based models like ALOCC and USAD do not effectively capture interdependencies in the data, resulting in weaker detection performance. (\romannumeral2) DAGMM, constrained by Gaussian Mixture Models, struggles with distribution estimation across multiple entities, and DeepSAD’s semi-supervised approach is limited by label quality, making it less effective in industrial contexts with label scarcity or noise. (\romannumeral3) Both GANF and MTGFlow, while performing better than some baselines, do not fully leverage interdependencies for anomaly detection. MADGA’s superior performance stems from its explicit use of interdependencies, allowing it to robustly detect deviations in MTS data.
\vspace{-0.2cm}


\begin{table}[t]
\caption{Ablation study for MADGA.}
\label{tab2:ablation}
\centering
\resizebox{!}{!}{
\begin{tabular}{c|ccccc}
\toprule
& PSM               & SWaT              & MSL               & SMD    &WADI           \\
\midrule
w/o GA        & 84.9±2.4          & 83.2±1.4          & 65.4±1.8          & 91.5±7.2      & 91.9±1.1    \\
w/o GWD        & 85.8±2.1          & 85.9±1.5          & 66.8±1.9          & 91.6±7.0          & 91.9±1.1  \\
w/o WD         & 86.9±1.5          & 88.3±1.1          & 67.6±1.7          & 91.8±6.7          & 92.1±1.0  \\
\textbf{MADGA} & \textbf{87.2±1.4} & \textbf{88.9±0.7} & \textbf{68.1±1.6} & \textbf{92.1±6.7} & \textbf{92.0±1.0} \\
\bottomrule
\end{tabular}}
\vspace{-0.5cm}
\end{table}

\begin{table}[t]
\caption{Hyper-parameter analysis for MADGA on SWAT.}
\label{tab2:parameter}
\centering
\resizebox{\linewidth}{!}{
\begin{tabular}{c|ccccc}
\toprule
{Batch \textbackslash Window}  & 40       & 60       & 80                & 100      & 120 \\
\midrule
64                                      & 83.5±2.0 & 83.8±2.3 & 83.4±2.3          & 83.9±2.6 & 85.7±1.8 \\
128                                     & 83.8±1.9 & 84.0±2.0 & 86.0±2.4          & 85.1±2.4 & 86.0±2.1 \\
256                                     & 84.0±1.6 & 84.4±1.8 & 86.7±1.9          & 87.0±1.7 & 86.0±1.7 \\
512                                     & 86.6±1.5 & 86.8±1.3 & \textbf{88.9±0.7} & 87.7±1.2 & 86.7±1.4 \\
1024                                    & 84.2±1.6 & 85.8±1.9 & 84.6±1.5          & 84.8±2.3 & 87.1±2.3 \\
\bottomrule
\end{tabular}}
\vspace{-0.3cm}
\end{table}

\begin{figure}[t]
\centering
\subfigure{\includegraphics[width=0.95\linewidth]{MADGA_graph.png}\label{add2:subfig1}}
\vspace{-0.2cm}
\subfigure{\includegraphics[width=0.95\linewidth]{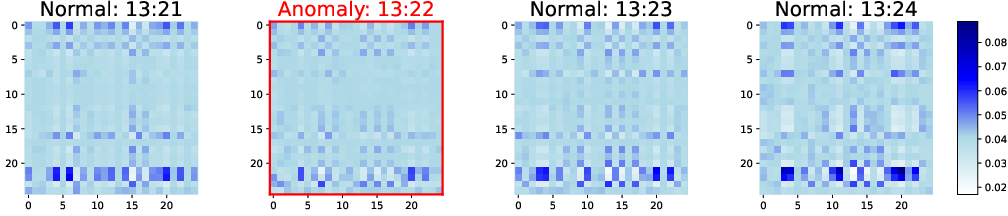}\label{add2:subfig2}}
\label{Interdependency Shift}
\caption{Interdependency Graph of MADGA (up) and w/o GA (down).}
\vspace{-0.4cm}
\end{figure}

\subsection{Architecture Analysis}

\subsubsection{Ablation Study}
We conducted an ablation study to evaluate the contributions of each component in MADGA, with results shown in Table \ref{tab2:ablation}. We tested configurations excluding node alignment (w/o WD), edge alignment (w/o GWD), and both (w/o GA). The results underscore the importance of both node and edge alignments for anomaly detection. The poorest performance was observed without both (w/o GA), highlighting the necessity of leveraging interdependencies. While MADGA without node alignment (w/o WD) performed better than without edge alignment (w/o GWD), the findings confirm that while node alignment handles distributional deviations, edge alignment is critical for capturing interdependencies. Combining both leads to optimal performance.

\subsubsection{Influence of the Batch Size and Window Length}
We explored the impact of batch size and window length, as summarized in Table \ref{tab2:parameter}. MADGA's performance was sensitive to batch size, with 512 identified as optimal for the SWaT dataset. Larger batches risked overfitting, positioning anomalies in high-density regions. Larger window sizes generally improved performance by allowing for more accurate distribution approximations, enhancing anomaly detection.
\vspace{-0.1cm}

\subsection{Distinguishing and Utilizing Interdependency}
We analyzed shifts in interdependencies from normal to anomalous states in the SWaT dataset, comparing MADGA with a variant without GA (w/o GA). Using self-attention, MADGA constructs dynamic graphs where the attention matrix serves as the adjacency matrix. Fig. \ref{Interdependency Shift} shows the learned edge weights over time, illustrating more pronounced shifts in interdependencies under MADGA than in the w/o GA variant. The distinct graph structures during normal conditions highlight GA's effectiveness in capturing and utilizing interdependencies. MADGA's edge alignment further enhances this distinction, improving anomaly detection by explicitly measuring interdependency differences. Future work will focus on refining interdependency variation measurements and edge alignment to better model complex data relationships.
\vspace{-0.2cm}

\section{Conclusion}
\label{section:Conclusion}
This work introduces MADGA, the first framework to transform MTS anomaly detection into a graph alignment problem by leveraging both distributional and interdependency aspects of channels. Using Wasserstein distance for node alignment and Gromov-Wasserstein distance for edge alignment, MADGA establishes a robust theoretical basis. Extensive experiments on real-world datasets demonstrate MADGA's superior performance in anomaly detection and interdependency analysis, highlighting the critical role of explicitly capturing interdependencies for identifying anomalies.

\section*{Acknowledgment}
\vspace{-0.2cm}
This work was supported by the National Natural Science Foundation of China under Grants (62201072, 62101064, 62171057, U23B2001, 62001054, 62071067), the Ministry of Education and China Mobile Joint Fund (MCM20200202, MCM20180101)

\bibliographystyle{IEEEtran}
\bibliography{reference.bib}

\end{document}